\documentclass[twoside]{article}

\usepackage[accepted]{aistats2020}
\usepackage[sort,numbers]{natbib}

\usepackage[utf8]{inputenc} 
\usepackage[T1]{fontenc}    
\usepackage{hyperref}       
\usepackage{url}            
\usepackage{booktabs}       
\usepackage{amsfonts}       
\usepackage{nicefrac}       
\usepackage{microtype}      
\usepackage{amsmath,amssymb,amsthm}
\usepackage{graphics,graphicx}
\usepackage{bm}
\usepackage{bbm}
\usepackage{color}
\usepackage{appendix}
\usepackage{algorithm} 
\usepackage{algorithmic} 
\usepackage[inline]{enumitem}
\usepackage{mathtools}
\usepackage{booktabs}
\usepackage{caption}
\usepackage{multirow}
\usepackage{subcaption}
\usepackage[inline]{enumitem}
 
\hypersetup{
     colorlinks = true,
      citecolor = blue,
}
 
\allowdisplaybreaks

\newcommand{\ie}{\emph{i.e.}} 
\newcommand{\eg}{\emph{e.g.}} 
\newcommand{\etc}{\emph{etc}} 

\newcommand{\indep}{\rotatebox[origin=c]{90}{$\models$}}

\newtheorem{theorem}{Theorem}
\newtheorem{definition}[theorem]{Definition}
\newtheorem{lemma}[theorem]{Lemma}
\newtheorem{corollary}[theorem]{Corollary}

\newtheorem*{remark*}{Remark}

\newcommand{\sgn}{\mathrm{sgn}}

\newcommand{\w}{\ensuremath \mathbf{w}}
\newcommand{\x}{\ensuremath \mathbf{x}}
\newcommand{\z}{\ensuremath \mathbf{z}}
\newcommand{\punt}[1]{}

\theoremstyle{remark}

\newcommand{\vz}{\mathbf{z}}

\newcommand{\bq}{\begin{equation}}
\newcommand{\eq}{\end{equation}}
\newcommand{\ba}{\begin{eqnarray}}
\newcommand{\ea}{\end{eqnarray}}

\newcommand{\remove}[1]{}

\newcommand{\Nrm}{\mathcal{N}}

\newcommand{\algoref}[1]{Algorithm~\ref{algo:#1}}  

\newcommand{\Dat}{\mathcal{D}}

\usepackage{color}  
\definecolor{gray}{rgb}{0.5, .5, .5}

\begin{document}

%
\runningtitle{Privacy-Preserving Inverse Probability Weighting}

%

\twocolumn[

\aistatstitle{Privacy-Preserving Causal Inference \\
via Inverse Probability Weighting}

\aistatsauthor{Si Kai Lee\textsuperscript{\normalfont 1}$^*$ \And Luigi Gresele \textsuperscript{\normalfont 1,2} \And Mijung Park \textsuperscript{\normalfont 1, 3} \And Krikamol Muandet\textsuperscript{\normalfont 1}}
\aistatsaddress{\textsuperscript{1}Max Planck Institute for Intelligent Systems, T\"ubingen, Germany\\
                \textsuperscript{2}Max Planck Institute for Biological Cybernetics, T\"ubingen, Germany\\
                \textsuperscript{3}University of T\"ubingen, T\"ubingen, Germany}]

\begin{abstract}
%
The use of inverse probability weighting (IPW) methods to estimate the causal effect of treatments from observational studies is widespread in econometrics, medicine and social sciences.
%
%
%
%
Although these studies often involve sensitive information, thus far there has been no work on privacy-preserving IPW methods. 
%
%
We address this by providing a novel framework for privacy-preserving IPW (PP-IPW) methods.
We include a theoretical analysis of the effects of our proposed privatisation procedure on the estimated average treatment effect, 
%
%
%
%
and evaluate our PP-IPW framework on synthetic, semi-synthetic and real datasets. 
The empirical results are consistent with our theoretical findings.
\end{abstract}

\section{INTRODUCTION}

The increasing ubiquity of machine learning in our daily lives has created a pressing need for trustworthy artificial intelligence (AI).
One key tenet of trustworthy AI, defined in the European Commission's Ethics Guidelines for Trustworthy AI \cite{EU19:Ethics}, is \emph{privacy}.
Preserving patient privacy is critical in medicine as it builds trust and fosters thoughtful decision making, which in turn helps improve patient care. 
Although the privacy requirements in the medical field are especially high, such requirements are not unique. 
In observational studies for social sciences, it is commonly assumed that sensitive information such as employment, education and criminal records that are used in analyses would be kept private. 
The datasets compiled from these studies often contain personal information, hence conclusions drawn from statistical analyses on such datasets runs the risk of violating the privacy of those present in the datasets.
Such risk extends to \emph{causal inference} methods~\cite{Glass13:CIPH,Imbens15:CIS,Shiffrin16:CIBig,Bareinboim16:Fusion} as they fall under the umbrella of statistical estimation tools.
   
We illustrate the problem of causal inference with an example.
In medicine, it is crucial to have a well-rounded understanding of the efficacy of different medical treatments since a treatment can produce a broad range of responses across patients and multiple treatments are usually available. 
The treatment effect can be modelled from observational data, by looking at how patients responded to different treatments in the past. However, two formidable obstacles need to be overcome in order to obtain a sound \emph{causal} effect estimation.
First, for each patient we only observe the outcome associated with the treatment that patient received, and no other~\cite{Holland86:FPCI}.
Second, the treatments that patients receive are not assigned at random, as doctors assign the treatment they expect to work best for each patient; as a result, treatment assignments and outcomes are subject to \emph{confounding}, which could result in a biased estimate of the outcome of each treatment.  
An accurate estimate of the treatment effect should take confounding into account, possibly modelling how individual characteristics of the patient determine the assigned treatment. 
However, this often requires collecting sensitive information from patients.

The \emph{propensity score} is arguably one of the most used quantities in causal inference for observational studies. 
It forms the basis of popular techniques such as matching, stratification, and inverse probability weighting (IPW) \cite{austin2011introduction, Rosenbaum83:PS, rosenbaum1984reducing, rosenbaum1985constructing} which are extensively used in econometrics, medicine, and social sciences
\cite{Imbens15:CIS}. 
Moreover, the IPW estimator based on propensity scores is the backbone of several counterfactual inference algorithms in the machine learning literature \citep{Dudik11:DoublyRobust,Bottou13:Counterfactual,Swaminathan15:CRM,Swaminathan17:Slate}.
Despite its widespread use, the propensity score---defined as the probability of assignment of a particular treatment given observed covariates---could depend on sensitive information about the patients, such as age, gender, race, and ethnicity.
Little concern has thus far been raised regarding the privacy issues related to the use of propensity scores in such methods.
Previously, propensity scores have been used by \cite{Rassen10:Pooling} as privatised substitutes for individual covariates. 
However, as we show in Section \ref{sec:det-ate}, this method still violates privacy \cite{dwork2014algorithmic} since sensitive observational data is used for \textit{estimating} propensity scores.


Inverse probability weighting (IPW) methods, which employs estimated propensity scores, are frequently used to estimate the average treatment effect (ATE) \cite{ robins2000marginal, rosenbaum1987model}. Since estimating the ATE with IPW methods still require sensitive observational data, privacy is further violated.

To address this important but neglected issue, we develop a novel framework for privacy-preserving IPW methods. This framework consists of two steps: (1) we learn a privacy-preserving propensity score estimator
and (2) we output a privacy-preserving ATE estimator. 
In addition, we investigate the effect of privatisation in both steps on the performance of the resulting causal analysis.

\vspace{-0.8em}
\paragraph{Related Work.} 
To the best of our knowledge, this paper is the first work which formally investigate the privatisation of the propensity score function and the average treatment effect estimates with IPW methods.
There have only been few prior attempts to privatise causal inference techniques in different contexts. 
For example, in \cite{KusnerSSW16}, the authors demonstrated how one could privatise statistical dependence scores such as \textit{Spearman's $\rho$} and \textit{Kendall's $\tau$} under the additive noise model.
The main focus of \cite{KusnerSSW16} is to obtain privatised scores that still can correctly identify the causal direction between two random variables.
In \cite{8166616}, the authors developed a differentially private constraint-based causal graph discovery method for categorical data. 
None of the above papers considered propensity score-based causal inference methods which is our focus.  
Since propensity score-based methods are fundamental to econometrics, medicine, social sciences etc., we expect this work to impact diverse fields.

\vspace{-0.8em}
\paragraph{Contributions.}
In this paper, we propose a privacy-preserving framework for IPW methods which comprises a propensity score estimator and an IPW-based average treatment effect estimator
We privatise the parameters of the logistic regression model used to estimate the propensity scores, as well as the output of inverse probability weighting. 
This guarantees the privacy of the individuals in the training dataset used to learn the propensity score estimator and individuals in the estimation dataset used to estimate the ATE. 
We analyse the \textit{effect of the noise added to enforce privacy} on the resulting estimated causal effect. 
Our analysis provides guidelines on how many samples we need to guarantee a certain level of privacy while providing accurate causal inference. 
We test our method on synthetic, semi-synthetic and real-world datasets to illustrate its effectiveness.

The rest of this paper is organised as follows. Section \ref{sec:background} provides a review of propensity scores in causal inference and differential privacy. 
The privacy-preserving propensity scores, IPW estimator, as well as their theoretical guarantees are then presented in Section \ref{sec:det-ate}, followed by experimental results in Section \ref{sec:experiments}. 
Finally, Section \ref{sec:conclusion} concludes the paper.

\section{BACKGROUND}
\label{sec:background}
In this section, we introduce relevant concepts from causal inference and differential privacy.
The key quantities are summarised in Table \ref{tab:notation}.

\begin{table}[t!]
    \centering
    \caption{Key Quantities}
    \label{tab:notation}
    \begin{tabular}{cl}
         \toprule
         {\small\textbf{Symbol}} & {\small\textbf{Description}} \\
         \midrule
         $\epsilon$ & Privacy loss \\
         $\delta$ & Failure probability \\
         $\mu_t$ & Population mean under treatment $t$ \\
         $\hat{\mu}_t$ &  Estimated $\mu_t$ \\
         $\hat{\mu}_t^{\epsilon}$ & Privatised $\hat{\mu}_t$ \\
         $\tau$ & Average Treatment Effect (ATE), $\mu_1 - \mu_0$\\
         $\hat{\tau}$ & ATE estimate, $\hat{\mu}_1 - \hat{\mu}_0$\\
         $\hat{\tau}_n$ & Partially privatised ATE estimate\\
         $\hat{\tau}_n^{\epsilon}$ & Fully privatised ATE estimate\\
         $\pi_{\w}$ & Propensity score function \\
         $\pi_{\hat{\w}}$ & Estimated $\pi_{\w}$ \\
         $\pi_{\hat{\w}}^{\epsilon}$ & Privatised $\pi_{\hat{\w}}$ \\
         \bottomrule
    \end{tabular}
\end{table}

\subsection{Propensity Scores in Causal Inference}
\label{sec:ppci}
The \emph{potential outcomes} framework is one of the most widely-used approaches in causal inference \cite{Neyman1923:Causal, Rubin74:Causal, Rubin05:PO}. 
It provides the mathematical basis for estimating the outcome of an experiment which has not been performed given outcomes observed under other experimental settings.

Consider the setting where we want to estimate whether a given treatment has a positive, negative or null effect on different units/individuals.
We define $T$ as the treatment variable and $Y_t$ as the random variable representing the potential outcome associated with treatment $T=t$.
In medicine, $T$ could represent different cancer treatments and $Y_t$ an indicator for patient recovery after treatment $t$. 
Throughout this paper, we focus on the binary treatment setting, \ie, $T\in\{0, 1\}$, and refer to the subset of the population with $T=1$ as the \textit{treatment} group, and the rest, with $T=0$, as the \textit{control} group. 
The random variables $Y_1$ and $Y_0$ are the outcomes associated with the treatment and control group, respectively. 

The question we want to answer is: \emph{what is the effect of administering a treatment to a unit compared to not doing so?}
Quantitatively, this can be characterised by the differences of the outcome when the treatment is administered and the outcome when the treatment is not administered, \ie, $Y_1 - Y_0$.
To estimate this, we would require both the outcomes of treatment and no treatment to be observed for every unit. 
However, for each unit we can observe only either $Y_1$ or $Y_0$. 
In practice, we substitute each unobserved quantity with an estimate of its expected outcome, $\mu_t := \mathbb{E}[Y_t]$, and evaluate the \textit{average treatment effect} (ATE) with $\tau = \mathbb{E}[Y_1] - \mathbb{E}[Y_0]$. 
Given a dataset $\mathcal{D}=\{(t_1,y_1),\ldots,(t_N,y_N)\}$, we can approximate $\hat{\mu}_t$ with $n_t^{-1}\sum_{i=1}^N \mathbbm{1}(t_i=t)\cdot y_i$, where $\mathbbm{1}(\cdot)$ is an indicator that returns 1 when $t_i=t$ and 0 otherwise.

If $\mathcal{D}$ is collected from a randomised experiment, $\hat{\tau} := \hat{\mu}_1-\hat{\mu}_0$ is an unbiased estimate of $\tau$. 
However, the problem with most observational studies is that $\hat{\tau}$ is generally biased because of potential \emph{confounding} variables $X$ that affect both $T$ and $Y_t$. 
For example, $X$ could be the current stage of a patient's cancer, which could influence both the decision of the physician regarding the treatment and the outcome of the treatment.
To obtain an unbiased ATE estimate despite the confounding variables $X$ for dataset $\mathcal{D}=\{(\x_i,t_i,y_i)\}_{i=1}^N$, we expand each $\mathbb{E}[Y_t]$ as $\mathbb{E}[Y_t] = \mathbb{E}_X[\mathbb{E}_{t}[Y_t|X,T=t]]$ and compute the difference of the quantity above with $t=1$ and with $t=0$. 
The validity of this estimate can be assessed given three technical requirements, which we assume throughout:
\begin{enumerate}[label={\bf (\roman*)}]
\item \textbf{Stable Unit Treatment Value Assumption (SUTVA):} The observed outcome of the $\mathrm{i}^{th}$ unit $Y(i)$ is unaffected by the assigned treatment to other units. 
\item \textbf{Ignorability:} $T \,\indep\, (Y_0,Y_1)\,|\, X$.
\item \textbf{Positivity:} $0 < \mathbb{P}(T=1|X=\x) < 1$ for all $\x$.
\end{enumerate}
See \cite{Imbens15:CIS} for a thorough exposition on these assumptions.

\paragraph{Propensity Scores.}
The \emph{propensity score} is one of the most widely used quantities for causal analysis in observational studies \cite{austin2011introduction, Rosenbaum83:PS, rosenbaum1984reducing, rosenbaum1985constructing}. 
In the binary treatment setting, the propensity score $\pi(\x)$ is defined as the probability of a unit with covariate $\x$ receiving treatment $T = 1$,
$\pi(\x) := \mathbb{P}(T = 1| X = \x)$.
It has been shown that, under the above set of assumptions, the propensity score $\pi(\x)$ summarises all the relevant information in $X$ for causal inference so that $T \,\indep\, (Y_0,Y_1)\,| \pi(X)$ holds \cite{Rosenbaum83:PS}. 
Unfortunately, in most observational studies, the true treatment assignment mechanism is not known. 
Thus, a common practice is to fit a propensity score function on data $\mathcal{D}$ using standard statistical models $\pi_\w(\x) = f_{\w}(\x)$, where $f_{\w}$ represents a model parameterised by a parameter vector $\w$. 
In this work, we focus on the logistic regression model since it is the most frequently used model for fitting propensity scores \cite{cepeda2003comparison}. 
The model is defined as 
\begin{align} \label{eq:logistic}
\pi_{\w}(\x) = \frac{1}{1+ e^{-\w^\top \x}}
\end{align}
\noindent where $\w \in\mathbb{R}^d$. 
Other popular techniques for propensity score estimation include classification and regression trees (CART), boosted CART, random forests, \etc.

\paragraph{Inverse Probability Weighting (IPW).}
\label{sec:IPW}
One popular propensity score-based method for estimating $\tau$ from observational data is IPW \cite{ robins2000marginal, rosenbaum1987model}. Using IPW, we can obtain an unbiased ATE estimate $\hat{\tau} = \hat{\mu}_1 - \hat{\mu}_0$ with $\hat{\mu}_1$ and $\hat{\mu}_0$ defined as
\begin{align}
\label{eq:iptw}
\hat{\mu}_1 := \frac{1}{N} \sum_{i=1}^N \frac{y_i t_i}{\pi(\x_i)}, \quad
\hat{\mu}_0 := \frac{1}{N} \sum_{i=1}^N \frac{y_i (1-t_i)}{1-\pi(\x_i)},
\end{align}
\noindent where $N$ is the number of units in $\mathcal{D}$. 
In practice, we replace the propensity score function $\pi$ with its empirical estimate $\pi_{\hat{\w}}$. 



\subsection{Differential Privacy}
The notion of differential privacy (DP) \cite{dwork2014algorithmic} provides a well-defined framework to describe the privacy properties of statistical estimation algorithms. 
DP states that a privacy-preserving, randomised algorithm behaves \textit{similarly} on \textit{similar} datasets. Specifically, the algorithm's behaviour is quantified in terms of a probability ratio, which describes how the algorithm's output changes when different datasets are used as an input. Intuitively, the probability ratio does not change much (behaves similarly) if the input datasets differ by a single entry (similar datasets). The formal definition  is given below.
\begin{definition}\label{def:dp}
A randomised algorithm $\mathcal{A}$ with domain $\mathbb{N}^{|\mathcal{X}|}$, where $\mathcal{X}$ is the data universe, satisfies $(\epsilon, \delta)$-differential privacy, i.e., is $(\epsilon, \delta)$-DP, if for all $\mathcal{S} \subseteq$ Range($\mathcal{A}$) and for all neighbouring $\mathcal{D}, \mathcal{D}' \in \mathbb{N}^{|\mathcal{X}|}$ such that $\|\mathcal{D} - \mathcal{D}'\|_1 \leq 1$, i.e., there is only one entry difference in the two datasets $\Dat, \Dat'$,
\begin{equation*}
\mathbb{P}[\mathcal{A}(\mathcal{D}) \in \mathcal{S}] \leq \exp(\epsilon)\mathbb{P}[\mathcal{A}(\mathcal{D}') \in \mathcal{S}] + \delta
\end{equation*} 
where the probability space is over the outputs of $\mathcal{A}$.
\end{definition} 
Here, $\epsilon$ is defined as the \textit{privacy loss} controlling the level of privacy. 
For $0<\delta<1$, $\delta$ defines the \textit{failure probability}, \ie, an algorithm is $\epsilon$-DP with probability at least $1-\delta$. 

\paragraph{Gaussian Mechanism.}
The Gaussian mechanism \cite{dwork2006our} is commonly used to privatise models (see, \eg, \cite{dwork2014algorithmic,SarwateC13} for other DP mechanisms). 
The mechanism privatises a vector-valued function $f: \Dat \mapsto \mathbb{R}^p$ by adding Gaussian noise to it. The noise is calibrated based on the L2-{\it{sensitivity}} of $f$, which is defined by $S(f) = \max_{\Dat, \Dat', \|\Dat-\Dat'\|=1} \|f(\Dat) - f(\Dat') \|_2$.
Hence, the privatised function has the form $\tilde{f}(\Dat) = f(\Dat) + \Nrm(0, \sigma^2 \mathbf{I}_d)$.
A choice of $\sigma \geq \epsilon^{-1}\sqrt{2\log(1.25/\delta)}S(f)$ produces the function $\tilde{f}(\Dat) $ that is $(\epsilon, \delta)$-DP.


\paragraph{Differentially Private Empirical Risk Minimisation (DP-ERM).}
Let $\ell:\mathbb{R}\times\mathbb{R}\to\mathbb{R}_+$ be the loss function and vector $\w$ the model parameters of $\pi_{\w}$.
Under ERM framework, the optimal model parameters $\hat{\w}$ are obtained by minimising the empirical risk function $J(\w, \mathcal{D}) = m^{-1} \sum_{i=1}^{m} \ell(\pi_\w(\x_i), t_i) + \lambda \Omega(\w)$, where $\Omega(\cdot)$ is the regulariser and $\lambda > 0$ the regularisation constant. 
If logistic regression \eqref{eq:logistic} is used to model the propensity score function $\pi_{\w}$, the parameters $\w$ can be learned using ERM with a $L_2$-regulariser.
We assume throughout that $\mathcal{X}$ is contained in the $L_2$-unit ball\footnote{This is a typical assumption on the dataset in the differential privacy literature.}, \ie, $\|\x_i\|_2 \leq 1$ for all $\x_i \in \mathcal{X}$. Note that the dataset $\mathcal{D}_m = \{(\x_i,t_i,y_i)\}_{i=1}^m$ contains the observed treatment effects $y_i$, but that they are not used to fit $\pi_{\w}$. 
In our case, the regularised cross-entropy loss $J(\w, \mathcal{D})$ is 
\begin{align}\label{eq:loss}
- \frac{1}{m} \sum_{i=1}^m t_i \log p_i + (1 - t_i)\log (1-p_i) +  \frac{\lambda}{2} \|\w\|_2^2
\end{align}
\noindent where $p_i := p(\pi_\w(\x_i))$. This loss is equivalent to the logistic loss used in \cite{chaudhuri2011differentially}. 
The L2-sensitivity of $\hat{\w}$ obtained by minimising \eqref{eq:loss} is 
given by $S(\hat{\w}) = \max_{\Dat, \Dat', \|\Dat-\Dat'\|=1} \left\|\hat{\w}(\Dat) - \hat{\w}(\Dat')
  \right\|_2 \leq 2(m\lambda)^{-1}$ \cite{chaudhuri2011differentially}.

\section{Privacy-Preserving IPW}
\label{sec:det-ate}

We start by explaining why logistic-regression based estimation for the propensity score by minimising \eqref{eq:loss} is not private.
%
%
An intuitive explanation is: since an ERM solution can be written as a linear combination of training samples, the $\hat{\w}$ and $\hat{\w}'$ estimated from datasets $\mathcal{D}$ and $\mathcal{D}'$ differing by one single entry could be completely different, \eg, if the entry was an outlier. 
Hence, the likelihood ratio of the two models $\hat{\w}$ and $\hat{\w}'$ estimated from two neighbouring datasets would be unbounded, which makes $\hat\w$ \textbf{\textit{not}} $\epsilon$-DP. 
A more rigorous proof can be found in~\cite{chaudhuri2011differentially}.

Since existing propensity score estimation often relies on ERM, 
standard methods for modelling the propensity score function would yield a model that violates differential privacy: given $\pi_{\hat{\w}}$ and all points in dataset $\mathcal{D}$ bar one, it is possible to infer the covariates $\x_i$ of omitted unit.
%
This is a serious problem since propensity score-based methods are frequently used to estimate the causal effect from observational studies containing sensitive data and highlights the need for privacy-preserving propensity score estimators and propensity score-based methods.

Privatising the ATE estimated via IPW requires two layers of privatisation.

The first step of our proposed method is to divide the dataset $\mathcal{D}$ into $\Dat_m = \{\x_i, t_i, y_i\}_{i=1}^m$ and $\Dat_n = \{\x_i, t_i, y_i\}_{i=1}^n$, and use the first $m$ points to learn $\hat{\w}$ and the remaining $n$ points to estimate $\hat{\tau}$. 
We separately privatise the estimated propensity score function with respect to the $m$ datapoints and the estimated ATE with respect to the $n$ datapoints to ensure that all $N$ datapoints (where $N=m+n$) in $\mathcal{D}$ are protected.

We describe each step in detail in the next two subsections.

\subsection{Privacy-Preserving Propensity Score}

To privatise the logistic regression model, we first compute a non-private version of the propensity score $\pi_{\hat{\w}}$ by minimising ~\eqref{eq:loss}. Next, we use the Gaussian mechanism to generate a privacy-preserving version of $\hat{\w}$, defined as $\hat{\w}_{\epsilon}$. 

\begin{definition}
\label{def:pps}
Let $\hat{\w}$ be the solution
of \eqref{eq:loss}. 
A privacy-preserving propensity score function is
\begin{align}\label{log_reg}
\pi_{\hat{\w}}^{\epsilon}(\x) = \frac{1}{1+ \exp(-\hat{\w}_{\epsilon}^\top \x)} = \frac{1}{1+ \exp(- \hat{\w}^\top \x - \z^\top \x)},
\end{align}
where $\hat{\w}_{\epsilon} := \hat{\w} + \z$ with $\z \sim \mathcal{N}(\mathbf{0}, \sigma^2\mathbf{I}_d)$ and $\sigma = \epsilon^{-1}\sqrt{2\log(1.25/\delta)}S(\hat{\w})$ for $\epsilon \in (0,1)$ and any $\delta \in (0,1)$. 
\end{definition}

Alongside Definition \ref{def:pps}, we define the counterparts of $\hat{\tau}$, $\hat{\mu}_1$ and $\hat{\mu}_0$ which use $\pi_{\hat{\w}}^\epsilon(\x)$ in place of $\pi_{\hat{\w}}(\x)$ over the $n$ points as $\hat{\tau}_n$, $\hat{\mu}_1^{\epsilon}$ and $\hat{\mu}_0^{\epsilon}$. This yields $$\mbox{\textbf{DP ATE  w.r.t. $\Dat_m$}}: \hat{\tau}_{n} := \hat{\mu}_1^{\epsilon} - \hat{\mu}_0^{\epsilon}$$ where 
\begin{align}
\label{eq:formpropen}
\hat{\mu}_1^{\epsilon} := \frac{1}{n}\sum_{i=1}^{n_{1}} \frac{y_i}{\pi_{\hat{\w}}^{\epsilon}(\x_i)}, \quad
\hat{\mu}_0^{\epsilon} := \frac{1}{n}\sum_{i=1}^{n_{0}} \frac{y_i}{1-\pi_{\hat{\w}}^{\epsilon}(\x_i)},
\end{align}
with the first sum is over the $n_{1}$ points where $t_i = 1$ and the second sum is over the $n_0$ points where $t_i = 0$, and $n_0+n_1 = n$. The estimator $\hat{\tau}_n$ only safeguards the privacy of $m$ points within the dataset, those belonging to the split $\Dat_m$. 
Privatisation of the points in $\Dat_n$ is discussed in the next subsection. Here, we focus on the effect of the noise added for protecting $\Dat_m$ on causal inference.

\begin{figure*}[t!]
    \centering
    \includegraphics[width=0.45\linewidth]{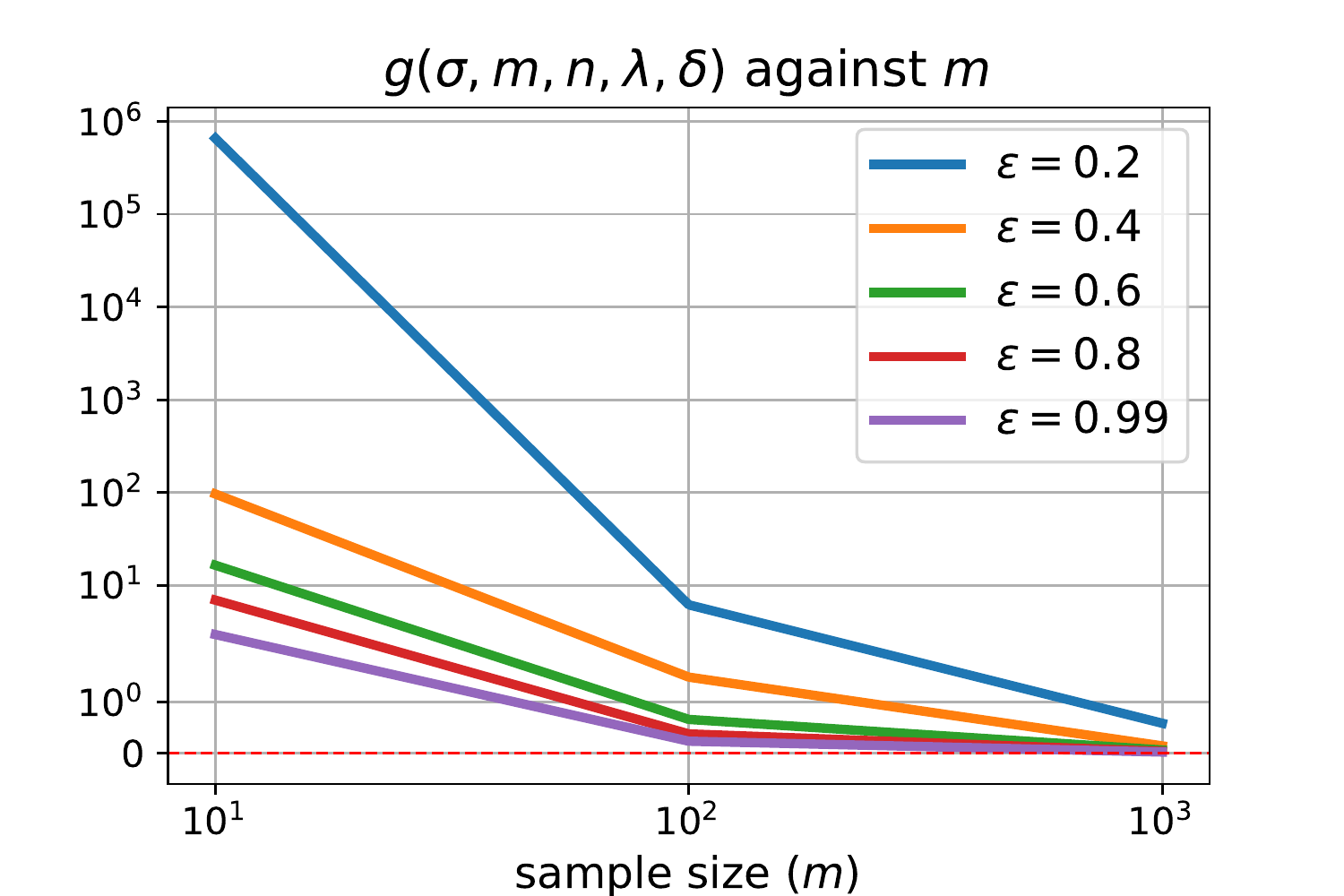}\hfill
    \includegraphics[width=0.45\linewidth]{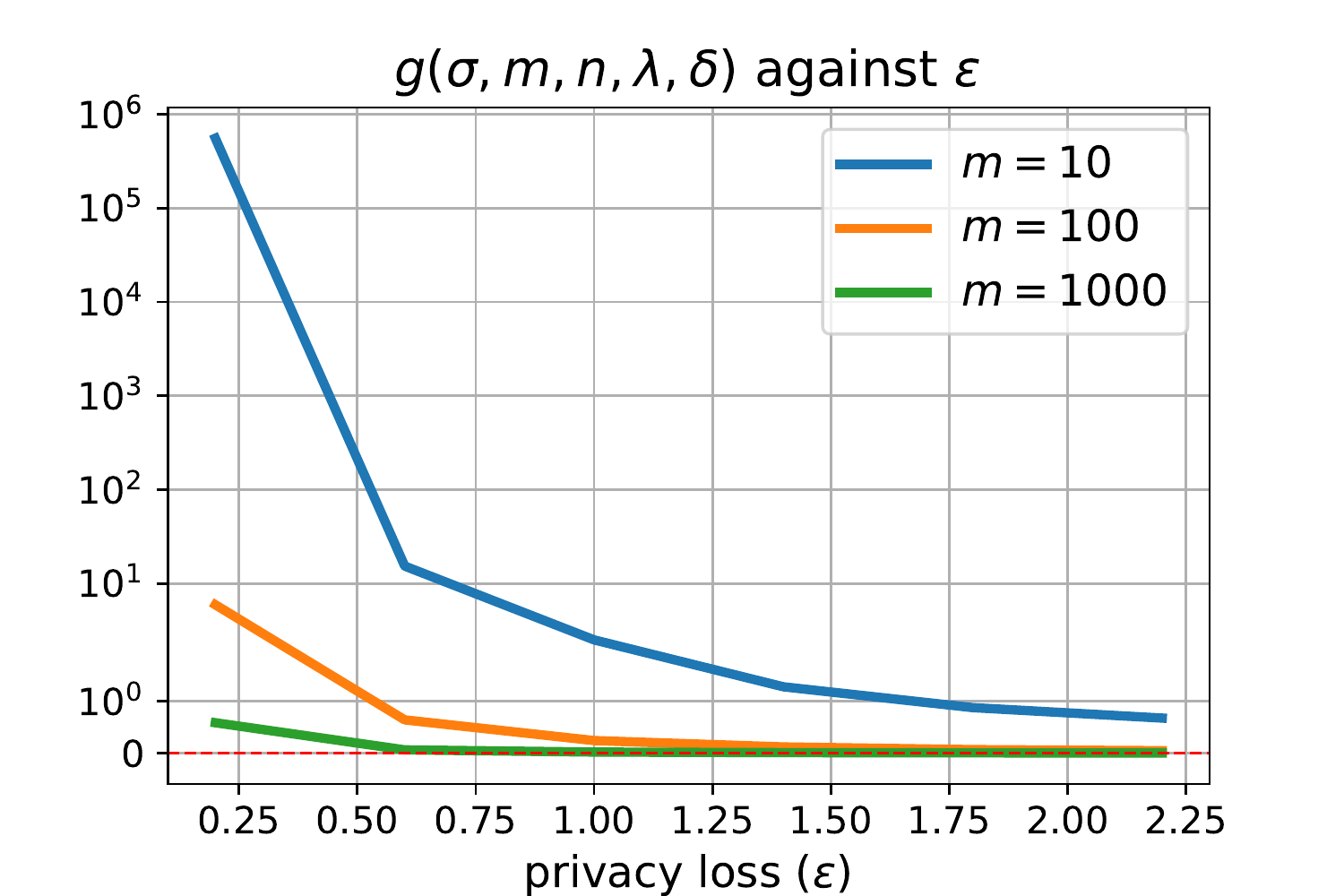}
    \caption{Behaviour of $g(\epsilon,m,n,\lambda,\delta)$ w.r.t. the sample size $m$ and the privacy loss $\epsilon$. As we can see, the bias decreases as we expect when the sample size $m$ and the privacy loss $\epsilon$ increase.}
    \label{fig:g_sigma}
\end{figure*}




\paragraph{Characterising $\hat{\tau}_n$.} Both $\hat{\mu}_1^{\epsilon}$ and $\hat{\mu}_{0}^{\epsilon}$ are weighted sums of correlated log-normal random variables, where the magnitude and signs of the weights depend on the data. This can be seen from their formulae
\begin{align}
\label{eq:mu0-mu1}
\hat{\mu}_1^{\epsilon} &= \frac{1}{n}\sum_{i=1}^{n_{1}} y_i (1 + \exp(-\w^T \x_i)\exp(-\z^T \x_i)),\\ 
\hat{\mu}_0^{\epsilon} &= \frac{1}{n}\sum_{i=1}^{n_{0}} y_i (1 + \exp(\w^T \x_i)\exp(\z^T \x_i)),
\end{align}
and recalling that $\z$ is a Gaussian random variable. 
Obtaining closed-form expressions for the above quantities with finite samples is highly nontrivial \cite{gulisashvili2016tail, lo2012sum}.

As a first step towards understanding the random variable $\hat{\tau}_n$, we show how its expected value can be rewritten in terms of the non-privatised ATE estimate $\hat{\tau}$ and variance of the added noise. See Appendix \ref{proof:exp-tau-hat} for proof.

\begin{lemma}\label{lem:exp-tau-hat}
    Let $\alpha_i := (-1)^{1-t_i}y_i\exp((-1)^{t_i}\hat{\w}^\top \x_i)$ be a constant for $i=1,\ldots,n$. 
    Then, we have $$\mathbb{E}[\hat{\tau}_n] = \hat{\tau} + g(\epsilon,m,n,\lambda,\delta)$$
    where the function $g(\epsilon,m,n,\lambda,\delta)$ is given by
    \begin{equation} \label{eq:g-func}
    \frac{1}{n}\sum_{i=1}^n \alpha_i\left[\exp\left((-1)^{t_i} \frac{4\log(1.25/\delta)\|\x_i\|_2^2}{\epsilon^{2}m^{2}\lambda^2}\right) - 1\right].
    \end{equation}
\end{lemma}



Lemma \ref{lem:exp-tau-hat} allows us to interpret $\hat{\tau}_{n}$ as a biased estimate of $\hat{\tau}$, where the \emph{additive} bias term is a function of the privacy loss $\epsilon$, sample sizes $m$ and $n$, regularisation constant $\lambda$, and failure probability $\delta$. 
Notice that the bias $g(\epsilon,m,n,\lambda,\delta)$ converges to zero as either the privacy budget $\epsilon$ or the total number of points $m+n$ goes to infinity as to be expected.
We complement our insights with numerical simulations (see Figure \ref{fig:g_sigma}) that describe the behaviour of this nontrivial bias term.

To study the probability of $\hat{\tau}_n$ having the opposite sign w.r.t. the non-private estimator, we assume that its support is bounded both from above and below. This allows us to employ standard concentration inequality results for variables with bounded supports. 
We bound the support of the estimator by either \emph{deterministic}, or \emph{probabilistic} means. An in-depth discussion of how we do so is provided in Appendix \ref{sec:det-prob-bounds}.

The next theorem characterises the behaviour of $\hat{\tau}_n$.

\begin{theorem}\label{thm:main-thm-1}
Assume that $\hat{\tau} > 0$ and $\mathrm{sign}(\hat{\tau})=\mathrm{sign}(\tau)$. If $|\hat{\tau}_n| \leq \eta$ for some $\eta > 0$, we have 
$$\mathbb{P}(\hat{\tau}_{n} \leq 0| \hat\tau>0) \leq \exp \left( -2\eta^{-2}(\hat{\tau} + g(\epsilon,m,n,\lambda,\delta))^2 \right).$$
\end{theorem}

\begin{proof}
Given that $|\hat{\tau}_{n}| \leq \eta$ (or the result in Lemma \ref{lem:prob-bdc} with probability at least $1-\gamma$), we can apply Lemma \ref{lem:exp-tau-hat} and Hoeffding's inequality for bounded variables to yield the result.
\end{proof}
Qualitatively, we expect this since the theorem implies that the larger the magnitude of the true estimated ATE is, the smaller the probability of drawing incorrect causal conclusions from the partially privatised estimator. Interestingly, the bound in Theorem \ref{thm:main-thm-1} provides a quantitative characterisation, showing that such probability decreases \emph{exponentially} as a function of $\hat{\tau}$. 
The probability of drawing incorrect conclusions from $\hat{\tau}_{n}$ also depends exponentially on the bias $g(\epsilon,m,n,\lambda,\delta)$. We provide empirical results clarifying the dependency on this term in Section \ref{sec:experiments}.

\begin{algorithm*}[t]
    \caption{PP-IPW}
    \begin{algorithmic}[1]
    \REQUIRE Data $\mathcal{D} = \{(\mathbf{x}_i,t_i,y_i)\}_{i=1}^N$ and privacy loss ($\epsilon, \delta$)
    \STATE Split $\mathcal{D}$ into two random subsets $\mathcal{D}_m$ and $\mathcal{D}_n$ consisting of $m$ and $n$ data points.
    \\
    \textbf{A. Obtain DP propensity score function} 
    \STATE Minimise $J(\w,\mathcal{D}_m)$ in \eqref{eq:loss} for the non-private estimate $\hat{\mathbf{w}}$. 
    \STATE $\hat\w_{\epsilon} =  \hat\w + \vz $ where $\z \sim \mathcal{N}(\mathbf{0}, \sigma_m^2\mathbf{I}_d)$ and $\sigma_m = \epsilon^{-1}\sqrt{2\log(1.25/\delta)}S(\hat{\w})$.
    \STATE Output DP propensity score function
    $\pi_{\hat{\w}}^{\epsilon}(\x) = 1/(1+ \exp(- \hat{\w}_{\epsilon}^\top \x))$. \\
    \textbf{B. Obtain DP ATE} 
    \STATE Compute $\hat{\tau}_{n} := \hat{\mu}_1^{\epsilon} - \hat{\mu}_0^{\epsilon}$, given $\mathcal{D}_n$, where 
    $\hat{\mu}_1^{\epsilon} := n^{-1}\sum_{i=1}^{n_{1}} \frac{y_i}{\pi_{\hat{\w}}^{\epsilon}(\x_i)}$ and
    $\hat{\mu}_0^{\epsilon} := n^{-1}\sum_{i=1}^{n_{0}} \frac{y_i}{1-\pi_{\hat{\w}}^{\epsilon}(\x_i)}$.
    \STATE Output DP ATE $\hat{\tau}_{n}^{\epsilon}= \hat{\tau}_{n} + e$, where $e \sim \mathcal{N}(0,\sigma_n^2)$ and $\sigma_n := \epsilon^{-1}\sqrt{2 \log (1.25 / \delta)} S(\hat{\tau}_{n})$.
    \ENSURE DP propensity score function $\pi_{\hat{\w}}^{\epsilon}$ (w.r.t. $\mathcal{D}_m$) and DP ATE  $\hat{\tau}_n^{\epsilon}$ (w.r.t $\mathcal{D}$).
    \end{algorithmic}\label{algo:PPIPW}
\end{algorithm*}

\subsection{Privacy-Preserving ATE} 

We now proceed to privatising the $n$ points used to compute $\hat{\tau}_n$. 
Given that $\mathbb{P}(T = 1 | x_i)$ is bounded above and below by $\omega_1$ and $\omega_2$ respectively for all $i \in n$, $\hat{\tau}_n$ is consequently bounded as well.
By further assuming that $|y_i| \leq C_y$ for all $i \in n$, we ensure that the L2-sensitivity of $\hat{\tau}_{n}$, $S(\hat{\tau}_{n})$, is bounded by $2 n^{-1} C_y \max[1/\omega_1, 1/(1 - \omega_2)]$. We show how this quantity is obtained in Appendix \ref{sec:l2-sens}.

We apply the Gaussian mechanism to $\hat{\tau}_{n}$ to obtain a privacy-preserving approximation $\hat{\tau}_{n}^{\epsilon}$ that is $(\epsilon, \delta)$-DP w.r.t to the $n$ points used to obtain the estimate. 
$$\mbox{\textbf{DP ATE w.r.t both $\Dat_m$ and $\Dat_n$} : }\hat{\tau}_{n}^{\epsilon} = \hat{\tau}_{n} + e $$ where the noise is drawn from $e \sim \mathcal{N}(0,\sigma_n^2)$ and with the noise standard deviation $\sigma_n := \epsilon^{-1}\sqrt{2 \log (1.25 / \delta)} S(\hat{\tau}_{n})$.

\paragraph{Characterising $\hat{\tau}^\epsilon_n$.}
We now present our main result which bounds the probability that \emph{both} $\hat{\tau}_{n}^{\epsilon}$ and $\hat{\tau}_n$ yield incorrect causal conclusions. See Appendix \ref{proof:main-thm-2} for proof.

\begin{theorem} \label{thm:main-thm-2}
Assume that $\hat{\tau} > 0$ and $\mathrm{sign}(\hat{\tau})=\mathrm{sign}(\tau)$. If $|\hat{\tau}_{n}| \leq \eta$ for some $\eta > 0$, we have 
\begin{align*}
\MoveEqLeft \mathbb{P}(\hat{\tau}_{n}^{\epsilon} \leq 0, \hat{\tau}_{n} \leq 0 | \hat\tau >0) \\
&\leq \frac{1}{2} \exp \left( \frac{-2(\hat{\tau} + g)^2}{\eta^2} \right)\left[1 + \mathrm{erf}\left(\frac{|\hat{\tau}_{n}|}{\sigma_n\sqrt{2}}\right)\right],
\end{align*} 
\noindent where $g := g(\epsilon,m,n,\lambda,\delta)$.
\end{theorem}

The additional $1/2 [1 + \mathrm{erf}(|\hat{\tau}_{n}|/\sigma_n\sqrt{2})]$ term in Theorem \ref{thm:main-thm-2} represents the probability that the added noise from the Gaussian mechanism $e$ is greater than $|\hat{\tau}_{n}|$, computed via the Gaussian CDF. The bound in Theorem \ref{thm:main-thm-2} further accounts for $\hat{\tau}_{n}$ and $\sigma_n$: the probability that both $\hat{\tau}_{n}^{\epsilon}$ and $\hat{\tau}_{n}$ are negative increases exponentially with $|\hat{\tau}_{n}|$ and decreases exponentially with $\sigma_n$ respectively where the latter is a function of $\epsilon$, $\delta$ and $S(\hat{\tau}_{n})$. 

This theorem provides a full characterisation of our proposed differentially private ATE estimation procedure. It summarises the effects of protecting all $m+n=N$ points, and show that the probability of drawing an incorrect causal conclusion decay exponentially w.r.t to the values of $\hat{\tau}$ and $\hat{\tau}_n$. 

\paragraph{Remark.}
Notice that in Theorem \ref{thm:main-thm-2} we chose to bound the probability $\mathbb{P}(\hat{\tau}_{n}^{\epsilon} \leq 0, \hat{\tau}_{n} \leq 0 | \hat\tau>0)$. However, depending on the application, other quantities, such as $\mathbb{P}(\hat{\tau}_{n}^{\epsilon} \leq 0 | \hat\tau>0)$ or $\mathbb{P}(\hat{\tau}_{n}^{\epsilon} \leq 0 | \hat{\tau}_{n} \leq 0, \hat\tau>0)$, might be more interesting.
It is also possible to bound those quantities; the proofs follow from that of Theorem \ref{thm:main-thm-2} presented in Appendix \ref{proof:main-thm-2}, and we omit them for brevity.

Putting everything together, our algorithm is presented in \algoref{PPIPW}. For the sake of simplicity, we assigned the same privacy budget for privatising the propensity score and the average treatment effect but one could choose to have separate privacy levels for the two quantities. We also include in Appendix \ref{sec:att_and_atc} bounds for the average treatment effect for the treated (ATT) and the average treatment effect for the controls (ATC).
We leave the development of privatised estimators based on more sophisticated techniques as future work.

\section{EXPERIMENTS} \label{sec:experiments}
In this section, we demonstrate our theoretical findings with experiments on synthetic, semi-synthetic and real data.
We set the probability of failure $\delta = 10^{-6}$ and regularisation coefficient $\lambda = 0.1$ for all experiments.
The logistic regression model is implemented in PyTorch \cite{paszke2017automatic} and optimised via gradient descent using the entire dataset.
To ensure reproducibility, we set the random seed for both NumPy and PyTorch to 1.

\begin{figure*}[t!]
\centering
    \begin{subfigure}{0.45\linewidth}
\includegraphics[width=\linewidth]{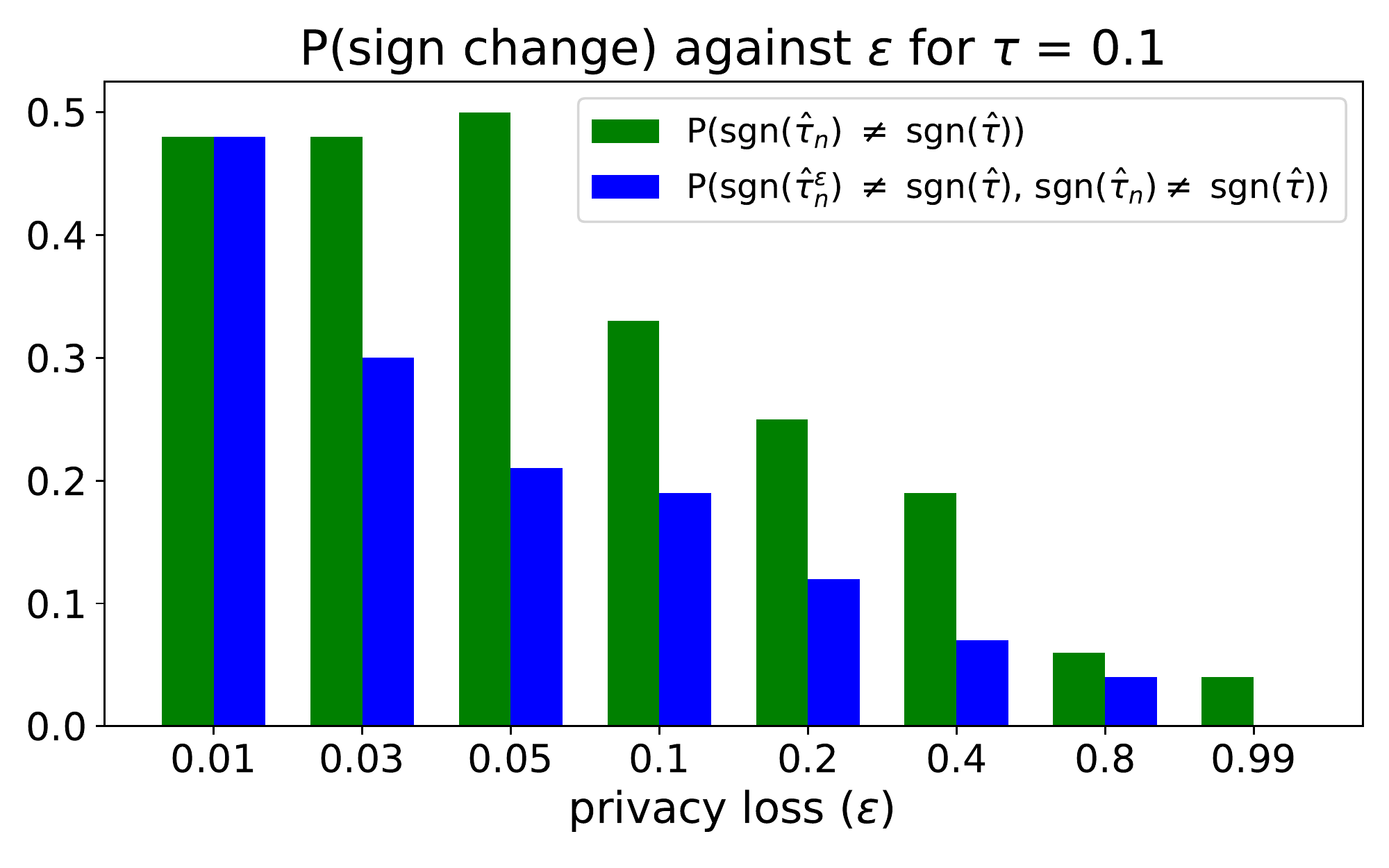} 
    \vspace{-0.75cm}
    \caption{}
\label{fig:1a}
    \end{subfigure}\hfill
    \begin{subfigure}{0.45\linewidth}
\includegraphics[width=\linewidth]{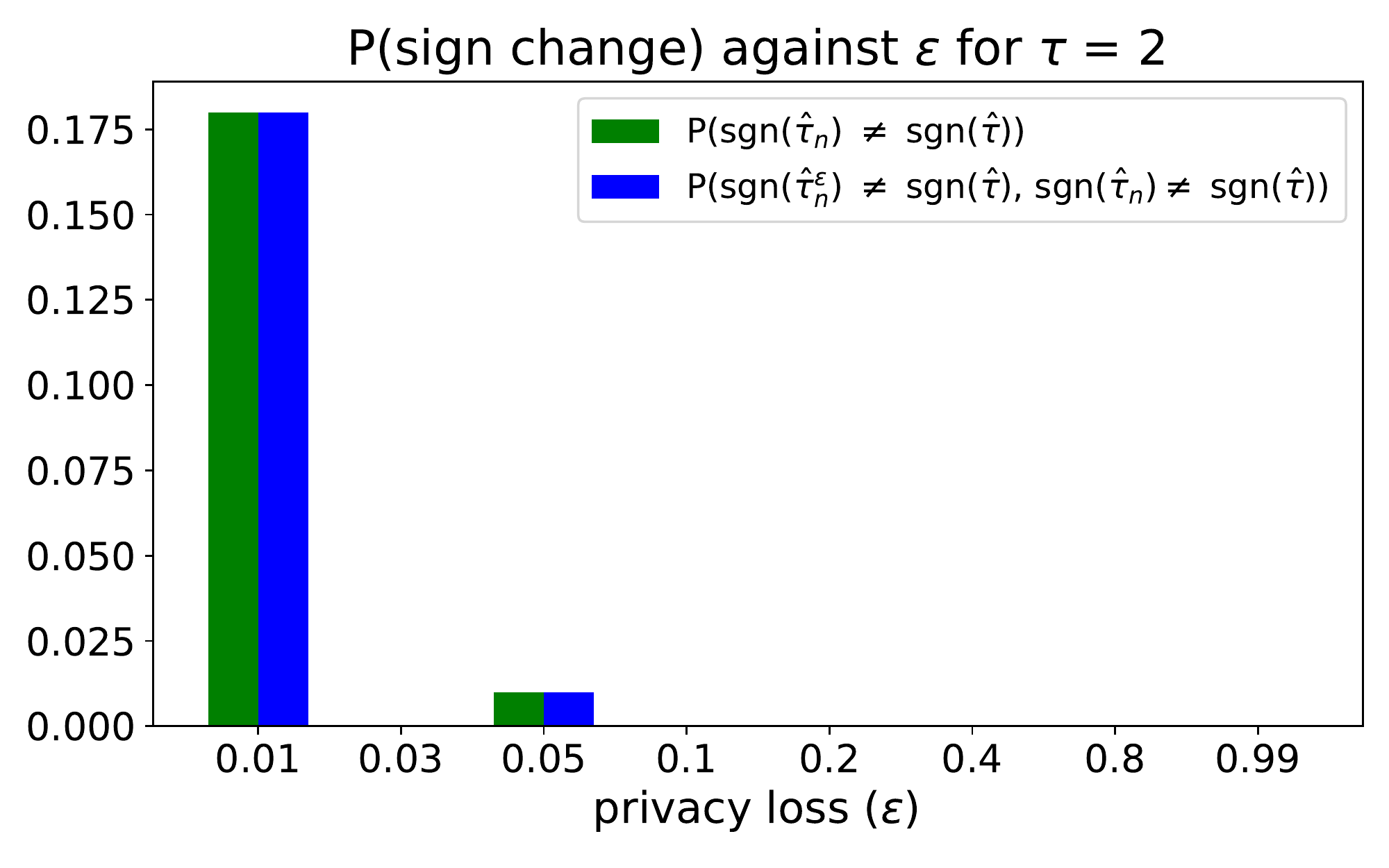}
    \vspace{-0.75cm}
    \caption{}
\label{fig:1b}
    \end{subfigure}
    \begin{subfigure}{0.45\linewidth}
    \includegraphics[width=\linewidth]{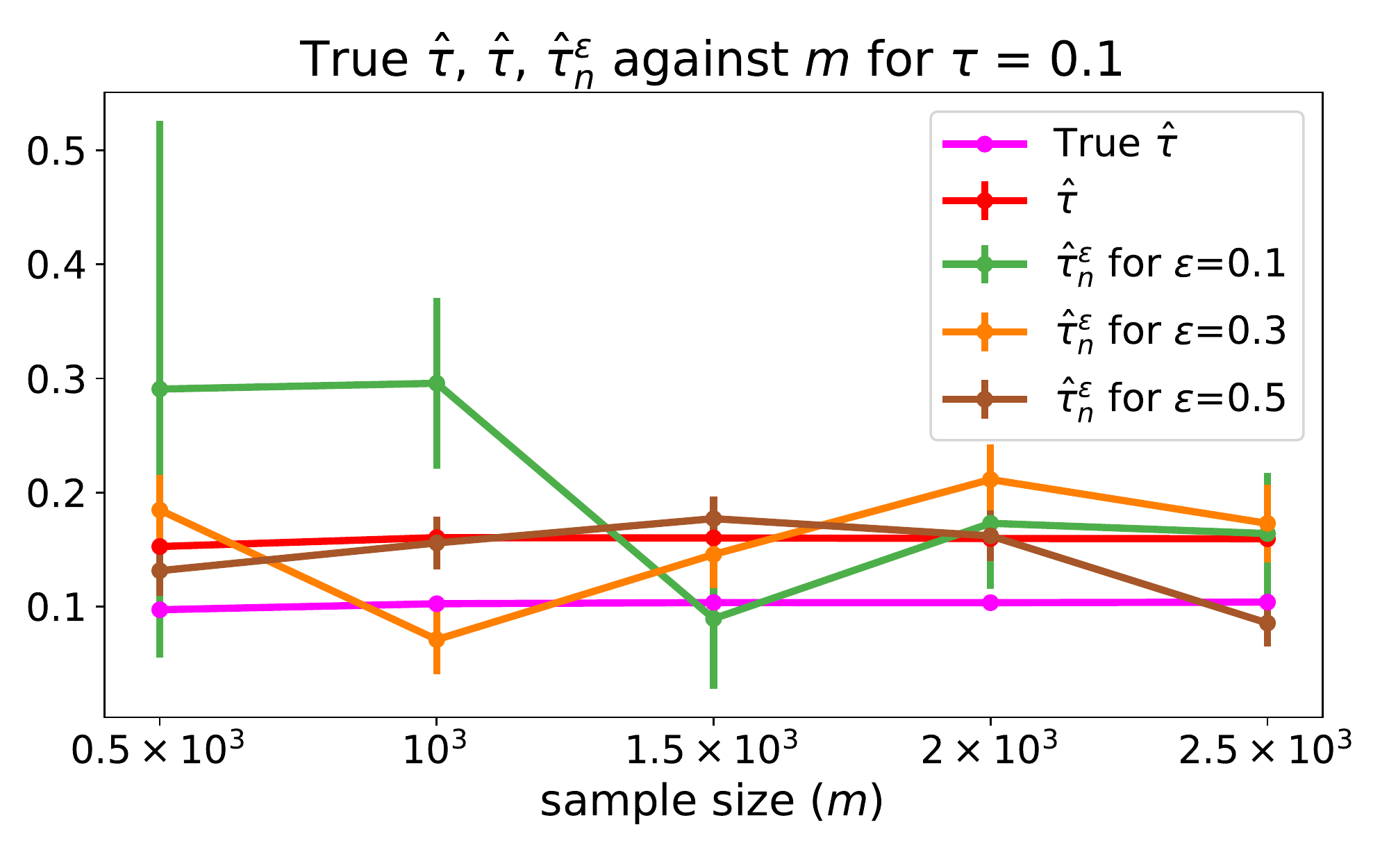}
    \vspace{-0.75cm}
    \caption{}
    \label{fig:1c}
    \end{subfigure}\hfill
    \begin{subfigure}{0.45\linewidth}
\includegraphics[width=\linewidth]{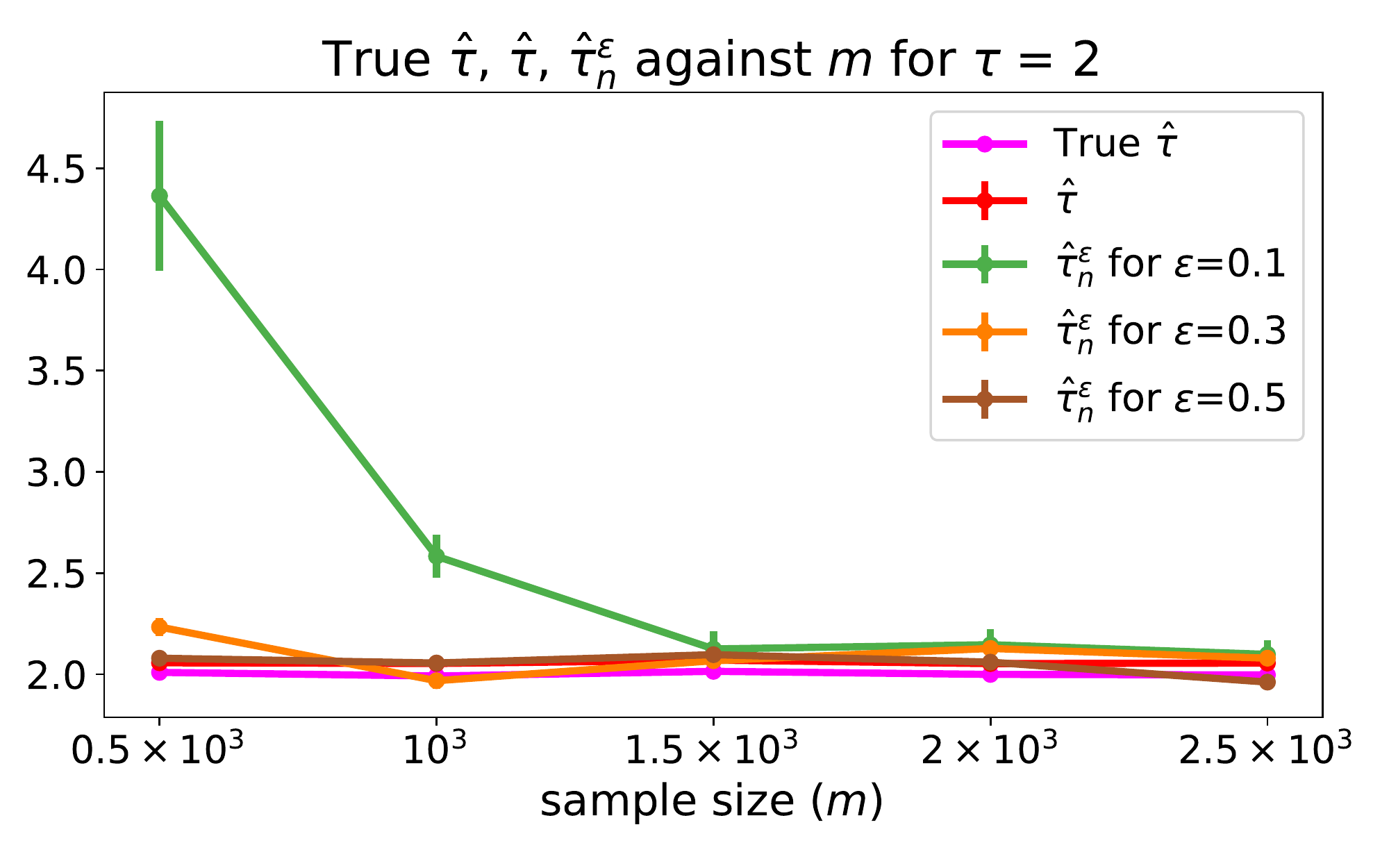}
    \vspace{-0.75cm}
    \caption{}
\label{fig:1d}
    \end{subfigure}
    \vspace{-0.25cm}
    \caption{Experimental results for synthetic data. (a) and (c) correspond to low-confidence with $\tau=0.1$. (b) and (d) correspond to the high-confidence with $\tau=2$. See main text for interpretation.}
    \label{fig:1}
\end{figure*}

\paragraph{Synthetic Data.}
For this set of experiments, we vary the $m$ points used to fit the logistic regression model, use $n = 1000$ points to estimate the ATE and average over 100 trials. 


We generate $X \in \mathbb{R}^{50}$ by sampling $\x_i$s separately from $\mathcal{N}(\mathbf{0}, 9 \cdot\mathbf{I}_{50})$ for each trial and standardising each separate sampled set of $\x_i$s with the maximum L2-norm of the $\x_i$s in the set. 
For each standardised $\x_i$, the treatment assignment $t_i$ and outcome $y_i$ are generated across trials in the following manner: 
\begin{align*}
    t_i &\sim \text{Bernoulli}(s(\mathbf{a}^\top\x_i)), & \mathbf{a} & \sim \mathcal{N}(\mathbf{0},\mathbf{I}_{50}),\\
    y_i &= \mathbf{b}^\top\x_i + t_i \tau + \vartheta, & \mathbf{b} & \sim \mathcal{N}(\mathbf{0},\mathbf{I}_{50})
\end{align*}
\noindent where $\vartheta\sim \mathcal{N}(0,0.01)$, $\tau\in\{0.1,2\}$ is a non-zero bias, and $s(\cdot)$ is a sigmoid function.
We perturb the weights of the learned logistic regression model with a sample from $\mathcal{N}(\mathbf{0}, \sigma^2\mathbf{I}_{50})$ with $\sigma$ defined in \eqref{log_reg}.

Figures \ref{fig:1a} and \ref{fig:1b} shows that the probability that the signs of $\hat{\tau}$ disagreeing with both $\hat{\tau}_n$ and $\hat{\tau}_n^\epsilon$ decreases as $\epsilon$ and $\tau$ increases with $m$ set to 1000. 
These results reflect the exponential dependence of $\mathbb{P}(\hat{\tau}_n \leq 0)$ and $\mathbb{P}(\hat{\tau}_n^\epsilon \leq 0, \hat{\tau}_n \leq 0)$ on $g(\epsilon,m,n,\lambda,\delta) = O(1/\epsilon^2)$ and $\hat{\tau}$ in Theorems \ref{thm:main-thm-1} and \ref{thm:main-thm-2}.

The convergence of $\hat{\tau}_n^\epsilon$ to $\hat{\tau}$ as $m$ increases from 500 to 2500 at intervals of 500 seen in Figures \ref{fig:1c} and \ref{fig:1d} reinforces the inverse relationship between $g(\epsilon,m,n,\lambda,\delta)$ and $m$ earlier demonstrated in Figure \ref{fig:g_sigma}. 
The error bars represent the 95\% confidence interval of the mean of the various estimates. 
The above plots also show that increasing $\tau$ yields exponentially larger mean ATE estimates which checks out with the form of $\mu_1^\epsilon$ and $\mu_0^\epsilon$ in \eqref{eq:mu0-mu1}, as the expectation of the log-normal random variables weighting $y_i$ monotonically increases with the variance of the Gaussian random variable.

\paragraph{Semi-synthetic Data.}
Next, we test our methods on the semi-synthetic binary-treatment Infant Health and Development Programme (IHDP) dataset that was introduced in \cite{hill2011bayesian}. 
We use the train and test sets from \cite{shalit2017estimating} for fitting logistic regression and estimating the ATE respectively with the true ATE of the train/test splits being 4.
IHDP is a real-world dataset with 25 covariates describing 747 children and their mothers, de-randomised binary treatments and synthetic continuous outcomes that can be used to compute a ground truth ATE \cite{hill2011bayesian}.
We create balanced training and ATE estimation datasets where $m = 500$ and $n = 500$ by sampling with replacement 250 units with $T =1$ and $T=0$ and 100 units with $T =1$ and $T=0$ respectively from the above train and test sets.
As the IHDP dataset comes with 1000 different realisations of train and test data, we average over all realisations.
In Tables \ref{tab:ihdp}, we see that increasing $\epsilon$ generally increases the fidelity of the mean ATE estimate and reduces $\mathbb{P}(\sgn(\hat{\tau_n}) \neq \sgn(\hat{\tau}))$ and $\mathbb{P}(\sgn(\hat{\tau_n^\epsilon}) \neq \sgn(\hat{\tau}))$ respectively. 
We did not include the average standard deviation of the estimates as estimating ATE with IPW is known to have high variance due to the unboundedness of the propensity score function $\pi_{\hat{\w}}$.
Lastly, another critical observation from Tables \ref{tab:ihdp}, especially for practitioners, is that too small an $\epsilon$ can lead to an unreliable estimate of $\hat{\tau}_{n}$ and $\hat{\tau}_n^{\epsilon}$, \ie, see the estimates for IHDP when $\epsilon=0.2$. 

\paragraph{Real Data.}
To further verify our proposed method, we use the Lalonde observational studies benchmark \cite{lalonde1986evaluating} obtained from \cite{cbps}. 
As we do not account for unbalanced datasets, we only use the original Lalonde dataset with 297 treated and 425 control individuals and subsample from it to create our training and ATE estimation datasets. 
There are 9 covariates containing sensitive information such as age, education, and race, and the outcome is 1978 earnings.
As no train/test splits are provided, we sample without replacement 100 units of $T=1$ and $T=0$ to create the ATE estimation dataset of size 200 and sample with replacement 250 units with $T=1$ and $T=0$ from the remaining points to generate the training dataset of size 500.
We do the above 1000 times to obtain the same number of realisations as the IHDP dataset.
The results for Lalonde in Tables \ref{tab:lalonde} supplements those for IHDP: increasing $\epsilon$ improves the accuracy of the mean ATE estimate and decreases $\mathbb{P}(\sgn(\hat{\tau_n}) \neq \sgn(\hat{\tau}))$ and $\mathbb{P}(\sgn(\hat{\tau_n^\epsilon}) \neq \sgn(\hat{\tau}))$.

\begin{table}[t!]
    \centering
    \caption{Average $\hat{\tau}_{n}$, $\hat{\tau}_n^{\epsilon}$, $\rho(\hat{\tau},\hat{\tau}_n) := \mathbb{P}(\sgn(\hat{\tau}_n) \neq \sgn(\hat{\tau}))$, and $\rho(\hat{\tau},\hat{\tau}_n^{\epsilon}) := \mathbb{P}(\sgn(\hat{\tau}_n^{\epsilon}) \neq \sgn(\hat{\tau}))$ for various $\epsilon$ over 1000 runs on IHDP dataset. 
    The average $\hat{\tau}$ is 4.80.}
    \label{tab:ihdp}
    \resizebox{\columnwidth}{!}{
    \begin{tabular}{cccccc}
        \toprule
        \multirow{2}{*}{\textbf{Estimate}} &
        \multicolumn{5}{c}{\textbf{Privacy loss} ($\epsilon$)} \\
        & $0.2$ & $0.4$
            & $0.6$ & $0.8$
            & $0.99$ \\ 
        \midrule\midrule
        $\hat{\tau}_n$ & -237487.30 & 32.34 & 8.37 & 5.22 & 5.31  \\
        $\hat{\tau}_n^{\epsilon}$  & -237477.13 & 32.20 & 8.50 & 5.09 & 5.40  \\
        \midrule
        $\rho(\hat{\tau},\hat{\tau}_n)$ & 0.515  & 0.425 & 0.344 & 0.287 & 0.229 \\   
        $\rho(\hat{\tau},\hat{\tau}_n^{\epsilon})$ & 0.517  & 0.425 & 0.347 & 0.301 & 0.228 \\
        \bottomrule
    \end{tabular}}
\end{table}

\begin{table}[t!]
    \centering
    \caption{Average $\hat{\tau}_{n}$, $\hat{\tau}_n^{\epsilon}$, $\rho(\hat{\tau},\hat{\tau}_n) := \mathbb{P}(\sgn(\hat{\tau}_n) \neq \sgn(\hat{\tau}))$, and $\rho(\hat{\tau},\hat{\tau}_n^{\epsilon}) := \mathbb{P}(\sgn(\hat{\tau}_n^{\epsilon}) \neq \sgn(\hat{\tau}))$ for various $\epsilon$ over 1000 runs on Lalonde dataset. The average $\hat{\tau}$ is 902.11.}
    \label{tab:lalonde}
    \resizebox{\columnwidth}{!}{
    \begin{tabular}{cccccc}
        \toprule
        \multirow{2}{*}{\textbf{Estimate}} &
        \multicolumn{5}{c}{\textbf{Privacy loss} ($\epsilon$)} \\
        & $0.2$ & $0.4$
            & $0.6$ & $0.8$
            & $0.99$ \\        
        \midrule\midrule
        $\hat{\tau}_n$ & 872.21 &  899.15 & 895.26 & 906.70 & 904.17  \\
        $\hat{\tau}_n^{\epsilon}$ & 876.83 &  897.75 & 893.37 & 903.95 & 900.25   \\
        \midrule
        $\rho(\hat{\tau},\hat{\tau}_n)$ & 0.143 & 0.072 & 0.049 & 0.027 & 0.035 \\   
        $\rho(\hat{\tau},\hat{\tau}_n^{\epsilon})$ & 0.154 & 0.071 & 0.05 & 0.039 & 0.034 \\
        \bottomrule
    \end{tabular}}
\end{table}

\section{CONCLUSION}
\label{sec:conclusion}
We proposed a differentially private IPW method for average treatment effect under the inverse probability weighting framework. 
A key element of our proposed method is the use of a newly defined private propensity score estimator. Unlike traditional propensity scores, ours can be deployed in causal analysis without running the risk of exposing the covariates of any unit used in estimating the propensity score function.
Furthermore, we demonstrate---both theoretically and empirically---that the ATE estimate resulting from an application of our method is consistent with its non-private counterpart with high probability.
In other words, the proposed propensity score function not only safeguards privacy, but also yields valid causal analyses with high probability.


We believe this work highlights long-neglected privacy concerns associated with the use of propensity scores in causal inference, but would also pave the way for subsequent developments at the intersection of differential privacy and causal inference.
Although the starting point of our work is a specific choice of a non-private method for ATE estimation, the analyses can be extended to more sophisticated estimators. 
In particular, we note that IPW is mathematically equivalent to other estimation methods such as stratification and the backdoor correction (see, \eg, ~\citep{Hernan19:CI}). 
Future extensions of our work will be dedicated to the development of privatised estimators based on alternative ATE estimation methods as well as other causal estimands such as conditional average treatment effect (CATE) \citep{shalit2017estimating}.
The effectiveness of private propensity scores in more complex methods and settings still remains an open question. 


\bibliographystyle{unsrt}
\bibliography{ref}

\appendix
\clearpage
\section{Proof of Lemma \ref{lem:exp-tau-hat}} \label{proof:exp-tau-hat}
\begin{proof}
Given a dataset $\mathcal{D} = \{(\x_i,t_i,y_i)\}_{i=1}^n$, let $\alpha_i$ be defined in Lemma \ref{lem:exp-tau-hat} and $\beta_i := \exp(\sigma^2\|\x_i\|_2^2/2)$ for $i=1,\ldots,n$.
Then, taking the expectation of \eqref{eq:mu0-mu1} w.r.t. the noise variable yields
\begin{align*}
\mathbb{E}[\hat{\mu}_1^{\epsilon}] &= \frac{1}{n} \sum_{t_i=1} y_i + y_i\beta_i\exp(-\hat{\w}^\top\x_i),\\ 
\mathbb{E}[\hat{\mu}_0^{\epsilon}] &= \frac{1}{n} \sum_{t_i=0} y_i + y_i\beta_i\exp(\hat{\w}^\top\x_i)
\end{align*}
We further rewrite $\mathbb{E}[\hat{\mu}_1^{\epsilon}]$ and  $\mathbb{E}[\hat{\mu}_0^{\epsilon}]$ as
\begin{align*}
\mathbb{E}[\hat{\mu}_1^{\epsilon}] 
&= \frac{1}{n}\sum_{t_i=1} \{y_i + y_i\exp(-\hat{\w}^\top\x_i)\}\ +\\
&\quad \ \frac{1}{n}\sum_{t_i=1} \{y_i\beta_i\exp(-\hat{\w}^\top\x_i) - y_i\exp(-\hat{\w}^\top\x_i)\} \\
&= \frac{1}{n}\sum_{t_i=1} y_i + y_i\exp(-\hat{\w}^\top\x_i)\ +\\
&\quad \ \frac{1}{n}\sum_{t_i=1} y_i\exp(-\hat{\w}^\top\x_i)(\beta_i-1) \\
&= \hat{\mu}_1 + \frac{1}{n}\sum_{t_i=1} y_i\exp(-\hat{\w}^\top\x_i)(\beta_i-1),\\
\mathbb{E}[\hat{\mu}_0^{\epsilon}] 
&= \hat{\mu}_0 + \frac{1}{n} \sum_{t_i=0} y_i\exp(\hat{\w}^\top\x_i)(\beta_i-1)
\end{align*}
Consequently,
\begin{align*}
\mathbb{E}[\hat{\tau}_{\epsilon}] 
&= \mathbb{E}[\hat{\mu}_1^{\epsilon}] - \mathbb{E}[\hat{\mu}_0^{\epsilon}]\\
&= \hat{\mu}_1 - \hat{\mu}_0 + \frac{1}{n} \sum_{i=1}^n \alpha_{i}(\beta_i - 1)\\
&= \hat{\tau} + g(\sigma)
\end{align*}
where $g(\sigma) := n^{-1}\sum_{i=1}^n \alpha_{i}(\beta_i - 1)$. Setting $\sigma = 2(\epsilon m \lambda)^{-1}\sqrt{2\log(1.25/\delta)}$ and substituting back into each $\beta_i$ yields the result \eqref{eq:g-func}.
\end{proof}

\section{Bounds for $\hat{\tau}_{n}$} \label{sec:det-prob-bounds}
\paragraph{Deterministic Bounds.}
The first case corresponds to using trimmed or clipped propensity scores. 
Given a constant $0<\xi<1$,  we define the trimmed version of $\hat{\tau}_{n}$ as
\begin{align*}
\hat{\tau}_{n,\xi} &:= \frac{1}{n}\sum_{i=1}^{n_{1}} \frac{y_i}{\max\{\xi,\pi_{\hat{\w}}^{\epsilon}(\x_i)\}}\ - \\
    &\ \ \quad \frac{1}{n}\sum_{i=1}^{n_{0}} \frac{y_i}{\max\{\xi,1-\pi_{\hat{\w}}^{\epsilon}(\x_i)\}}\,.
\end{align*}
While $\hat{\tau}_{n,\xi}$ is a biased estimate of $\tau$ with a bounded variance, it is often preferred due to its robustness to outliers. 
If $|y_i| \leq C_y$ for all $i \in n$, it follows that $|\hat{\tau}_{n,\xi}| \leq 2C_y\xi^{-1}$ with probability 1.

\paragraph{Probabilistic Bounds.} 
In the second case, we consider what happens if no trimming is applied.
Although the variance of $\hat{\tau}_{n}$ can be unbounded, it is a deterministic function of a single sub-Gaussian noise variable $\bm{z}\sim\mathcal{N}(\mathbf{0},\sigma^2 \mathbf{I}_d)$. 
Hence, we expect the bounded difference condition for $\hat{\tau}_{n}$ to hold \emph{with high probability}. 
To this end, let $\mathbb{S} := \sum_{j=1}^d z_j$. 
Since each component of $\z$ is independent, we have $\mathbb{S}\sim\mathcal{N}(0,d\sigma^2)$.
With $\mathbb{S}$ being a sub-Gaussian random variable, Chernoff's inequality \cite[pp. 21]{Boucheron13:CI} gives $\mathbb{P}(|\mathbb{S}| \geq \zeta) \leq 2\exp\left(-\zeta^2(2d\sigma^2)^{-1}\right)$ for some $\zeta > 0$.
This implies that $|\mathbb{S}| \leq \zeta$ holds with probability at least $1-\gamma$ where $\gamma = 2\exp(-\zeta^2(2d\sigma^2)^{-1})$.

The following lemma gives the probabilistic bounded difference condition for $\hat{\tau}_{n}$.
\begin{lemma}
\label{lem:prob-bdc}
Let $\hat{\tau}_{n}$ and $\hat{\tau}'_{n}$ be two estimates with different noise vectors $\z$ and $\z'$, respectively. 
Then, with probability at least $1-\gamma$, we have $|\hat{\tau}_{n} - \hat{\tau}'_{n}| \leq 
\eta$ where 
\begin{align*}
\eta
:=  \frac{2}{n} \sinh(\zeta) \left( \sum_{i=1}^{n_{1}} y_i \exp(-\hat{\w}^\top \x_i) + \sum_{i=1}^{n_{0}} y_i \exp(\hat{\w}^\top \x_i) \right) . 
\end{align*}
\end{lemma}

\begin{proof}
Let $\phi(\z)$ be the deterministic function mapping random variable $\z$ to $\hat{\tau}_{n}$ defined in \eqref{eq:formpropen}. Furthermore, let $C^1_i := y_i\exp(-\hat{\w}^\top\x_i)$ for $i \in n_1$ and $C^0_i := y_i\exp(\hat{\w}^\top\x_i)$ for $i \in n_0$.
Given that $|\mathbb{S}| \leq \zeta$ holds with probability at least $1-\gamma$, it follows that
\begin{align*}
&\left|\phi(\z) - \phi(\z')\right|\\
&= \frac{1}{n} \sum_{i=1}^{n_{1}} C^1_i \left|\exp(-\z^\top \x_i) - \exp(-\z'^\top \x_i)\right| +\\ 
&\quad \ \frac{1}{n} \sum_{i=1}^{n_{0}} C^0_i \left|\exp(\z^\top \x_i) - \exp(\z'^\top \x_i)\right| \\
&\leq \frac{\exp(\zeta) - \exp(-\zeta)}{n} \left( \sum_{i=1}^{n_{1}} C^1_i + \sum_{i=1}^{n_{0}} C^0_i \right)\\
&= \frac{2}{n} \sinh(\zeta) \left( \sum_{i=1}^{n_{1}} C^1_i + \sum_{i=1}^{n_{0}} C^0_i \right)
=: \eta ,
\end{align*}
\noindent also holds with probability at least $1-\gamma$. This concludes the proof. 
\end{proof}


\section{L2-sensitivity of $\hat{\tau}_n$} \label{sec:l2-sens}
\begin{proof}
If $|y_i| \leq C_y$ and $\Omega_1 < \mathbb{P}(T = 1 \,|\, \x_i) < \Omega_2$ for all $i \in n$ where $C_y$, $\Omega_1$ and $\Omega_2$ are constants, we have
\begin{align*}
&S(\hat{\tau}_n)\\
&= \max_{\substack{\mathcal{D}, \mathcal{D}' \\ \|\mathcal{D} - \mathcal{D}'\|_1 = 1}} \left\|\frac{1}{n} \sum_{i=1}^{n} \frac{y_i t_i}{\pi_{\hat{\w}}(\x_i)} - \frac{1}{n}\sum_{i=1}^{n}\frac{y_i (1-t_i)}{1-\pi_{\hat{\w}}(\x_i)}\right\|_2 \\
&= \frac{1}{n} \max \Bigg\{ \max_{\substack{\{\x_i, y_i, t_i\}, \{\x_i', y_i', t_i'\},\\ t_i, t_i' = 1}} \left\| \frac{y_i }{\pi_{\hat{\w}}(\x_i)} - \frac{y_i'}{\pi_{\hat{\w}}(\x_i)} \right\|_2, \\ 
&\quad \quad \max_{\substack{\{\x_i, y_i, t_i\}, \{\x_i', y_i', t_i'\}, \\ t_i, t_i' = 0}} \left\| \frac{y_i}{1 - \pi_{\hat{\w}}(\x_i)} - \frac{y_i'}{1 - \pi_{\hat{\w}}(\x_i)} \right\|_2 \Bigg\}\\
&\leq \frac{1}{n} \max \Bigg\{ \max_{\{\x_i, y_i, t_i\}, \{\x_i', y_i', t_i'\}, t_i, t_i' = 1} \left\| \frac{y_i }{\Omega_1} - \frac{y_i'}{\Omega_1} \right\|_2, \\ 
&\quad \quad \quad \max_{\{\x_i, y_i, t_i\}, \{\x_i', y_i', t_i'\}, t_i, t_i' = 0} \left\| \frac{y_i}{1 - \Omega_2} - \frac{y_i'}{1 - \Omega_2} \right\|_2 \Bigg\}\\
&\leq \frac{2 C_y}{n} \max \left\{ \frac{1}{\Omega_1}, \frac{1}{1 - \Omega_2} \right\}.
\end{align*}
Therefore, the L2-sensitivity of $\hat{\tau}_n$ is bounded from above by $2 n^{-1} C_y \max \left\{ \frac{1}{\Omega_1}, \frac{1}{1 - \Omega_2} \right\}$.
\end{proof}

\section{Error of $\hat{\tau}_n$} 
\label{sec:error}
This corollary bounds the error we incur by privatising the propensity score function.

\begin{corollary}
    For any constant $\Delta > 0$, 
    \begin{align*}
        &\mathbb{P}(|\hat{\tau}_{n} - \hat{\tau}| \geq \Delta)\\
        &\leq \frac{1}{n\Delta}\left|\sum_{i=1}^n \alpha_i\left[\exp\left({(-1)^{t_i}} \frac{4\log(1.25/\delta)\|\x_i\|_2^2}{\epsilon^{2}m^{2}\lambda^2}\right) - 1\right]\right|.
    \end{align*}
\end{corollary}

\begin{proof}
    By Lemma 3 and the Markov inequality,
    \begin{align*}
        &\mathbb{P}(|\hat{\tau}_{\epsilon} - \hat{\tau}| \geq \Delta) \\ 
        &\leq \frac{\left|\mathbb{E}[\hat{\tau}_{\epsilon}-\hat{\tau}]\right|}{\Delta} \\
        &= \frac{\left|\hat{\tau} + g(\epsilon,m,n,\lambda,\delta) - \hat{\tau}\right|}{\Delta} \\
        &= \frac{\left|g(\epsilon,m,n,\lambda,\delta)\right|}{\Delta} \\
        &= \frac{1}{n\Delta}\left|\sum_{i=1}^n \alpha_i\left[\exp\left((-1)^{t_i}\frac{4\log(1.25/\delta)\|\x_i\|_2^2}{\epsilon^{2}m^{2}\lambda^2}\right) - 1\right]\right|.
    \end{align*}
    The last step follows from the definition of $g(\epsilon,m,n,\lambda,\delta)$. This concludes the proof.
\end{proof}

\section{Proof of Theorem \ref{thm:main-thm-2}} \label{proof:main-thm-2}
\begin{proof}
    By the law of probability, 
    \begin{align*}
    \MoveEqLeft\mathbb{P}(\hat{\tau}_{n}^{\epsilon} < 0, \hat{\tau}_{n} < 0| \hat\tau>0) \\ 
    &= \mathbb{P}(\hat{\tau}_{n}^{\epsilon} < 0 | \hat{\tau}_{n} < 0, \hat\tau>0) \mathbb{P}(\hat{\tau}_{n} < 0| \hat\tau>0).
    \end{align*}
    Since we already possess the upper bound on $\mathbb{P}(\hat{\tau}_{n} < 0| \hat\tau>0)$ (see Theorem \ref{thm:main-thm-1}), we focus on obtaining $\mathbb{P}(\hat{\tau}_{n}^{\epsilon} < 0 | \hat{\tau}_{n} < 0)$. Since $\hat{\tau}_{n}^{\epsilon}$ is normally distributed with mean $\hat{\tau}_{n}$, this probability is just the probability that $\hat{\tau}_{n}^{\epsilon} < 0$ given that its mean is negative.
    Hence, we can get the exact probability using the Gaussian CDF, \ie,
    \begin{align*}
    \mathbb{P}(\hat{\tau}_{n}^{\epsilon} < 0 \,|\, \hat{\tau}_{n} < 0, \hat\tau>0) &= \Phi\left(\frac{|\hat{\tau}_{n}|}{\sigma_n}\right) \\
    &= \frac{1}{2}\left[1+\text{erf}\left(\frac{|\hat{\tau}_{\epsilon}|}{\sigma_n\sqrt{2}}\right)\right],
    \end{align*}
    where $\Phi(\cdot)$ denotes the CDF of the standard normal distribution and $\text{erf}(\cdot)$ is the error function.
    Combining this with the bound in Theorem \ref{thm:main-thm-1} yields
    \begin{align*}
    &\mathbb{P}(\hat{\tau}_{n}^{\epsilon} \leq 0, \hat{\tau}_{n} \leq 0 \,|\, \hat\tau>0)\\
    &\leq \frac{1}{2} \exp \left( \frac{-2(\hat{\tau} + g)^2}{\Delta^2} \right) \left[1 + \mathrm{erf}\left(\frac{|\hat{\tau}_n|}{\sigma_n\sqrt{2}}\right)\right],
    \end{align*}
    where $g := g(\epsilon,m,n,\lambda,\delta)$. This concludes the proof.
\end{proof}

\section{Privatised ATT and ATC Estimates} 
\label{sec:att_and_atc}
The IPW estimators for ATT and ATC have the form
\begin{align*}
\hat{\tau}_{\textrm{ATT}} 
&= \frac{1}{n}\sum_{i=1}^n \left(t_i - (1-t_i)\frac{\pi_{\hat{\w}}(\x_i)}{1-\pi_{\hat{\w}}(\x_i)}\right)y_i,\\
\hat{\tau}_{\textrm{ATC}}
&= \frac{1}{n}\sum_{i=1}^n\left(t_i \frac{1-\pi_{\hat{\w}}(x_i)}{\pi_{\hat{\w}}(x_i)} - (1-t_i)\right)y_i
\end{align*}and all related quantities are denoted with a subcripted ATT or ATC. First, we simplify on the $\frac{\pi(x_i)}{1-\pi(x_i)}$ and $\frac{1-\pi(x_i)}{\pi(x_i)}$ terms in $\hat{\tau}_{\text{ATT}}$ and $\hat{\tau}_{\text{ATC}}$ to obtain:
\begin{align*}
\hat{\tau}_{\text{ATT}} 
&= \frac{1}{n}\sum_{i=1}^n\left(t_i - (1-t_i) \exp(\hat{\w}^\top \x_i) \right)y_i,\\
\hat{\tau}_{\text{ATC}}
&= \frac{1}{n}\sum_{i=1}^n\left(t_i \exp(-\hat{\w}^\top \x_i) 
- (1-t_i)\right)y_i
\end{align*}
By employing the privacy-preserving propensity scores from Definition \ref{def:pps}, the perturbed $\hat{\tau}_{n, \text{ATT}}$ and $\hat{\tau}_{n, \text{ATC}}$ would have additional $\exp(\z^\top \x_i)$ and $\exp(-\z^\top \x_i)$ multiplicative terms on top of $\exp(\hat{\w}^\top \x_i)$ and $\exp(-\hat{\w}^\top \x_i)$ respectively with $\z \sim \mathcal{N}(\mathbf{0}, \sigma^2\mathbf{I}_d)$ and $\sigma = \epsilon^{-1}\sqrt{2\log(1.25/\delta)}S(\hat{\w})$ for $\epsilon \in (0,1)$ and any $\delta \in (0,1)$. This yields estimates that are DP w.r.t $\mathcal{D}_m$.  

We then adapt Lemma \ref{lem:exp-tau-hat} to ATT and ATC by adding and subtracting $\mu_{0, \textrm{ATT}}$ and $\mu_{1, \textrm{ATC}}$ to the expectation of the ATT and ATC respectively. Let $\alpha_{i, \textrm{ATT}} := - \mathbbm{1}_{t_i = 0}\ y_i \exp(\w^\top \x_i)$ and $\alpha_{i, \textrm{ATC}} := \mathbbm{1}_{t_i = 1}\ y_i \exp(-\w^\top \x_i)$ where $\mathbbm{1}$ is the indicator function and its subscript the condition where the function is 1. With in place, we obtain
\begin{align*}
&\mathbb{E}[\hat{\tau}_{n, \textrm{ATT}}] 
= \hat{\tau}_{\textrm{ATT}} + g_{\textrm{ATT}}(\epsilon,m,n,\lambda,\delta) \textrm{ where}\\
&g_{\textrm{ATT}}(\epsilon,m,n,\lambda,\delta) \\
&=\frac{1}{n} \sum_{i=1}^{n} \alpha_{i, \textrm{ATT}} \left[\exp\left((-1)^{t_i} \frac{4\log(1.25/\delta)\|\x_i\|_2^2}{\epsilon^{2}m^{2}\lambda^2}\right) - 1\right],
\end{align*}
and
\begin{align*}
&\mathbb{E}[\hat{\tau}_{n, \textrm{ATC}}] 
= \hat{\tau}_{\textrm{ATC}} + g_{\textrm{ATC}}(\epsilon,m,n,\lambda,\delta) \textrm{ where}\\
&g_{\textrm{ATC}}(\epsilon,m,n,\lambda,\delta)\\
&=\frac{1}{n} \sum_{i=1}^{n} \alpha_{i, \textrm{ATC}} \left[\exp\left((-1)^{t_i} \frac{4\log(1.25/\delta)\|\x_i\|_2^2}{\epsilon^{2}m^{2}\lambda^2}\right) - 1\right].
\end{align*}
We bound the supports of $\hat{\tau}_{n, \textrm{ATT}}$ and $\hat{\tau}_{n, \textrm{ATC}}$ by ensuring that $\exp(-\hat{\w}^\top \x_i)$ and $\exp(\hat{\w}^\top \x_i)$ are smaller or equal to some constant $\xi$ for all $i \in n$. This can be achieved by the techniques described in Appendix \ref{sec:det-prob-bounds}. With that in place, we can then extend Theorem \ref{thm:main-thm-1} to cover the cases where $\hat{\tau}_{\textrm{ATT}}$ and $\hat{\tau}_{\textrm{ATC}}$ are both assumed to be greater than 0 and have the correct sign as their true counterparts. The next two theorems illustrate the behaviour of  $\hat{\tau}_{n, \textrm{ATT}}$ and $\hat{\tau}_{n, \textrm{ATC}}$ respectively.
\begin{theorem}\label{thm:att_partial}
Assume that $\hat{\tau}_{\textrm{ATT}} > 0$ and $\mathrm{sign}(\hat{\tau}_{\textrm{ATT}})=\mathrm{sign}(\tau_{\textrm{ATT}})$. If $|\hat{\tau}_{n, \textrm{ATT}}| \leq \eta$ for some $\eta > 0$, we have
\begin{align*}
&\mathbb{P}(\hat{\tau}_{n, \textrm{ATT}} \leq 0| \hat\tau_{\textrm{ATT}}>0) \\
&\leq \exp \left( -2\eta^{-2}(\hat{\tau}_{\textrm{ATT}} + g_{\textrm{ATT}}(\epsilon,m,n,\lambda,\delta))^2 \right).
\end{align*}
\end{theorem}
\begin{theorem}\label{thm:atc_partial}
Assume that $\hat{\tau}_{\textrm{ATC}} > 0$ and $\mathrm{sign}(\hat{\tau}_{\textrm{ATC}})=\mathrm{sign}(\tau_{\textrm{ATC}})$. If $|\hat{\tau}_{n, \textrm{ATC}}| \leq \eta$ for some $\eta > 0$, we have
\begin{align*}
&\mathbb{P}(\hat{\tau}_{n, \textrm{ATC}} \leq 0| \hat\tau_{\textrm{ATC}}>0) \\
&\leq \exp \left( -2\eta^{-2}(\hat{\tau}_{\textrm{ATC}} + g_{\textrm{ATC}}(\epsilon,m,n,\lambda,\delta))^2 \right).
\end{align*}
\end{theorem}
Unsurprisingly, the results of Theorems \ref{thm:att_partial} and \ref{thm:atc_partial} have similar implications as Theorem \ref{thm:main-thm-1} for the ATE: the probability of drawing incorrect conclusions is exponentially related to the true estimated ATT/ATC and its corresponding bias $g_{ATT}$/$g_{ATC}$.

Continuing with our analysis, we apply the Gaussian mechanism to $\hat{\tau}_{n, \textrm{ATT}}$ and $\hat{\tau}_{n, \textrm{ATC}}$ to ensure that the estimates are DP w.r.t. the remaining $\mathcal{D}_n$ points.

We first conduct a standard sensitivity analysis on $\hat{\tau}_{n, \textrm{ATT}}$ and $\hat{\tau}_{n, \textrm{ATC}}$ to compute the variance of the noise added to the estimates using the Gaussian mechanism. Assuming that $|y_i| \leq C_{y}$ and $\exp(-\hat{\w}^\top \x_i) \leq \xi$ and $\exp(\hat{\w}^\top \x_i) \leq \xi$ for all $i \in n$ without loss of generality, we have 
\begin{align*}
S(\hat{\tau}_{n, \textrm{ATT}}) =  S(\hat{\tau}_{n, \textrm{ATC}}) \leq \frac{2C_y}{n}\max\{1, \xi\}
\end{align*}

We then fully privatise the ATT and ATC estimates DP w.r.t to both $\mathcal{D}_m$ and $\mathcal{D}_n$ by adding noise $e$ to $\hat{\tau}_{n, \textrm{ATT}}$ and $\hat{\tau}_{n, \textrm{ATC}}$ where $e \sim \mathcal{N}(0,\sigma_n^2)$ and the noise standard deviation $\sigma_{n, \textrm{ATT/ATC}}  := \epsilon^{-1}\sqrt{2 \log (1.25 / \delta)} S(\hat{\tau}_{n, \textrm{ATT/ATC}})$. 

The following theorems bounds bounds the probability that \emph{both} $\hat{\tau}_{n, \textrm{ATT}}^{\epsilon}$ and $\hat{\tau}_{n, \textrm{ATT}}$/$\hat{\tau}_{n, \textrm{ATC}}^{\epsilon}$ and $\hat{\tau}_{n, \textrm{ATC}}$ yield incorrect causal conclusions. The proofs are modified from that in Appendix \ref{proof:main-thm-2}.

\begin{theorem} \label{thm:main-thm-3}
Assume that $\hat{\tau}_{\textrm{ATT}} > 0$ and $\mathrm{sign}(\hat{\tau}_{\textrm{ATT}})=\mathrm{sign}(\tau_\textrm{ATT})$. If $|\hat{\tau}_{\textrm{ATT}}| \leq \eta$ for some $\eta > 0$, we have 
\begin{align*}
\MoveEqLeft \mathbb{P}(\hat{\tau}_{n, \textrm{ATT}}^{\epsilon} \leq 0, \hat{\tau}_{n, \textrm{ATT}} \leq 0 | \hat\tau_\textrm{ATT} >0) \\
&\leq \frac{1}{2} \exp \left( \frac{-2(\hat{\tau}_\textrm{ATT} + g)^2}{\eta^2} \right)\left[1 + \mathrm{erf}\left(\frac{|\hat{\tau}_{n, \textrm{ATT}}|}{\sigma_{n, \textrm{ATT}}\sqrt{2}}\right)\right],
\end{align*} 
\noindent where $g := g_\textrm{ATT}(\epsilon,m,n,\lambda,\delta)$.
\end{theorem}

\begin{theorem} \label{thm:main-thm-4}
Assume that $\hat{\tau}_{\textrm{ATC}} > 0$ and $\mathrm{sign}(\hat{\tau}_{\textrm{ATC}})=\mathrm{sign}(\tau_\textrm{ATC})$. If $|\hat{\tau}_{\textrm{ATC}}| \leq \eta$ for some $\eta > 0$, we have 
\begin{align*}
\MoveEqLeft \mathbb{P}(\hat{\tau}_{n, \textrm{ATC}}^{\epsilon} \leq 0, \hat{\tau}_{n, \textrm{ATC}} \leq 0 | \hat\tau_\textrm{ATC} >0) \\
&\leq \frac{1}{2} \exp \left( \frac{-2(\hat{\tau}_\textrm{ATC} + g)^2}{\eta^2} \right)\left[1 + \mathrm{erf}\left(\frac{|\hat{\tau}_{n, \textrm{ATC}}|}{\sigma_{n, \textrm{ATC}}\sqrt{2}}\right)\right],
\end{align*} 
\noindent where $g := g_\textrm{ATC}(\epsilon,m,n,\lambda,\delta)$.
\end{theorem}

Like Theorems \ref{thm:att_partial} and \ref{thm:atc_partial}, Theorems \ref{thm:main-thm-3} and \ref{thm:main-thm-4} provide similar insights to those obtained from \ref{thm:main-thm-2}. The probability that both $\hat{\tau}_{n, \textrm{ATT/ATC}}^{\epsilon}$ and $\hat{\tau}_{n, \textrm{ATT/ATC}}$ are negative increases exponentially with $|\hat{\tau}_{n, \textrm{ATT/ATC}}|$ and decreases exponentially with $\sigma_{n, \textrm{ATT/ATC}}$ respectively.

\end{document}